\documentclass[twoside,11pt,theapa]{article}
\usepackage{jair}
\usepackage{theapa}
\usepackage{amssymb, amsmath, amsthm}
\usepackage{algorithm}
\usepackage{algpseudocode}
\usepackage{verbatim}
\usepackage{graphics}
\usepackage{psfrag}

\numberwithin{equation}{section}

\newcommand\OR{\,\vee\,}
\newcommand\AND{\,\wedge\,}

\newcommand\NOT{\,\neg\,}

\newcommand\intersect{\, \cap\, }

\newcommand{\mc}[1]{{\cal #1}}

\newcommand{\de}{\delta}

\theoremstyle{plain}
\newtheorem{theorem}{Theorem}[section]
\newtheorem{lemma}[theorem]{Lemma}

\newtheorem{corollary}[theorem]{Corollary}

\theoremstyle{definition}
\newtheorem{definition}{Definition}[section]
\newtheorem{example}{Example}[section]

\theoremstyle{remark}
\newtheorem*{remark}{\textbf{Remark}}

\newtheoremstyle{break}
  {9pt}
  {9pt}
  {\itshape}
  {}
  {\bfseries}
  {.}
  {\newline}
  {}

\theoremstyle{break}

\newtheorem*{problem}{Problem}

\title{Fault Tolerant Boolean Satisfiability}
\author{\name Amitabha Roy \email aroy@cs.bc.edu\\
\addr Computer Science Department, \\
Boston College,
Chestnut Hill, 
MA 02467.}

\jairheading{25}{2006}{503--527}{6/04}{4/06}

\ShortHeadings{Fault Tolerant Boolean Satisfiability}
{Roy}
\firstpageno{503}

\begin{document}
\maketitle
\begin{abstract} 
  A $\delta$-model is a satisfying assignment of a Boolean formula for
  which any small alteration, such as a single bit flip, can be
  repaired by flips to some small number of other bits, yielding a new
  satisfying assignment.  These satisfying assignments represent
  \emph{robust} solutions to optimization problems (e.g., scheduling)
  where it is possible to recover from unforeseen events (e.g., a
  resource becoming unavailable).  The concept of $\delta$-models was
  introduced by \citeA{gins}, where it was
  proved that finding $\delta$-models for general Boolean formulas is
  NP-complete.  In this paper, we extend that result by studying the
  complexity of finding $\delta$-models for classes of Boolean
  formulas which are known to have polynomial time satisfiability
  solvers.  In particular, we examine 2-SAT, Horn-SAT, Affine-SAT,
  dual-Horn-SAT, $0$-valid and $1$-valid SAT.  We see a wide variation in
  the complexity of finding $\delta$-models, e.g., while 2-SAT and
  Affine-SAT have polynomial time tests for $\delta$-models, testing
  whether a Horn-SAT formula has one is NP-complete.
 \end{abstract}

 \section{Introduction}\label{sec:introduction}

 An important problem in the artificial intelligence community
 concerns the allocation of resources at or near the minimal cost. An
 optimal solution to such a problem might be rendered infeasible due
 to some unforeseen event (for example, a resource becoming
 unavailable or a task exceeding its allocated deadline). Hence, the motivation
 is to search for optimal solutions which are immune from such events.
 In this paper, we consider the complexity of finding such ``robust''
 solutions, where we only allow for a fixed small number of bad
 events, with the added condition that such bad events can be
 rectified by making a small change to the solution.  These solutions,
 which we call $\delta$-models, were introduced by  \citeA{gins},
 and further explored in \citeA{BM99}.  This approach to fault
 tolerance has been extended to constraint-satisfaction problems
 (CSPs) \cite{hebrard04,hebrardec04} and to  applications in
 combinatorial auctions~\cite{holland04}. \citeA{hoos}
 consider this approach to robustness in the framework of dynamic satisfiability (which
 they call DynSAT) where the goal is to be able to revise optimal
 solutions under a constantly changing input problem.

 We extend the initial complexity results in \citeA{gins} by looking at
 the theoretical complexity of tractable instances of satisfiability
 (SAT) identified by Schaefer's \emph{dichotomy} theorem \cite{Sch78}.
 The dichotomy theorem proves that the polynomial time solvable
 instances of SAT are $2$-SAT, Horn-SAT, dual-Horn-SAT, Affine-SAT,
 $0$-valid SAT and $1$-valid SAT and any other form is NP-complete.
 Our goal is to study the complexity of finding $\de$-models for the tractable
 problems identified by the dichotomy theorem.  We show a wide variation in complexity by type ($2$-SAT vs
 Horn-SAT) and by parameter (the number of repairs allowed for each
 break).

 Formally, a $\delta$-model of a Boolean formula, called supermodels
 by \citeA{gins}, is a satisfying assignment (satisfying assignments are usually called \emph{models}) such that if any bit of
 the assignment is flipped (from $0$ to $1$ or vice versa), one of the
 following conditions hold:
 \begin{itemize}
 \item[(i)] either the new assignment is a model or
 \item[(ii)] there is at least one other bit that can be flipped to
   obtain another model.
 \end{itemize}
 Flipping a bit of a $\delta$-model is called a  \emph{break}, corresponding
 to a ``bad'' event. The bit that is flipped to get another satisfying assignment is a \emph{repair}
 (we allow that some breaks may not need a repair). We also study a generalization of the concept:
 $\delta(r,s)$-models are satisfying assignments for which breaks to every set of up to
 $r$ bits need up to $s$ repairs (to avoid trivialities, we require that the repair bits are different from 
 the break bits).

 We let $\delta$-SAT refer to the decision question as to whether an
 input Boolean formula has a $\delta$-model.  When we restrict the
 form of the input Boolean formula, we refer to the corresponding
 decision questions as $\delta$-2-SAT, $\delta$-Horn-SAT etc. The
 higher degree variants of these problems are $\delta(r,s)$-SAT
 etc. where we consider $r$ and $s$ to be fixed integers.  The following problems
 are proved to be  NP-complete:
 \begin{itemize}
 \item[-] $\delta(r,s)$-SAT  \cite{gins},  $\delta(1,s)$-2-SAT for $s>1$,  
 \item[-] $\delta(1,s)$-Horn-SAT,  $\delta(1,s)$-dual-Horn-SAT, 
 \item[-] $\delta(r,s)$-$0$-valid-SAT and $\delta(r,s)$-$1$-valid-SAT.
 \end{itemize}
 In contrast, we prove that the  following problems are in P:  
 \begin{itemize}
 \item[-] $\delta(1,1)$-2-SAT, $\delta(r,s)$-Affine-SAT.
\end{itemize}

 The definition of $\delta$-models does not require that the new model
 obtained by repairing a break to a $\delta$-model is itself a
 $\delta$-model.   We define $\delta^*$-models to be $\delta$-models
 such that every break needs at most one repair to obtain another
 $\delta$-model. Such models represent the greatest degree of fault tolerance
 that can be achieved for the problem. We refer to the corresponding
 decision problems as $\delta^*$-SAT, $\delta^*$-2-SAT etc.
 We prove that  $\delta^*$-SAT is in NEXP (non-deterministic exponential
     time) and is NP-hard,  $\delta^*$-2-SAT is in P and that  $\delta^*$-Affine-SAT is in P.
\medskip

\begin{remark} Since our goal in this paper is to study the problems
  in Schaefer's tractable class with respect to fault tolerance, our
  yardstick to measure complexity is membership in P.  Hence, we do
  not concern ourselves with finding the exact running times within P.
  Optimizing runtimes may well prove important for practical applications (at least in
  the rare instances when we find polynomial time algorithms).
\end{remark}

\noindent\emph{Organization of the paper:} In Section~\ref{sec:defin-notat}, we introduce
and define the problem and establish notation. In Section~\ref{sec:compl-find-delta},
we study the complexity of finding $\delta$-models of general Boolean formulas.
In Section~\ref{section-restrict}, we consider the complexity of finding $\delta$-models
for restricted classes of formulas: we consider $2$-SAT (Section~\ref{subsection-2sat}),
Horn-SAT (Section~\ref{subsection-horn}), $0$-valid-SAT, $1$-valid-SAT (Section~\ref{subsection-valid})
and Affine-SAT (Section~\ref{sec:finding-de-models-affine}). We conclude with a section on future work (Section~\ref{section-future}).

\section{Definitions and Notations}  \label{sec:defin-notat}

   In this section, we establish some of the notation used in
 the rest of the paper and formally define the problems
 we wish to study.  
 
\bigskip

A Boolean variable can take on two values -- true or false which  we  write as $1$ and $0$ respectively.
A literal is either a variable $v$ or its negation, denoted by $\neg v$
(a variable is often called a \emph{pure literal}).  A clause is a
disjunction ($\OR$) of literals (for example, $v_1 \OR \neg v_2 \OR
v_3$ is a clause).  A Boolean formula  is a function from some
set  of Boolean variables $V=\{v_1, v_2, \ldots, v_n\}$ to $\{0,1\}$. In
computational problems, we assume that Boolean formulas are input in
a  canonical fashion: usually as a conjunction ($\AND$)
of clauses (in which case, we say that they are in conjunctive normal
form (CNF)).

We consider various forms of CNF formulas. A $2$-SAT formula is a
Boolean formula in CNF with at most $2$ literals per clause (more generally,
a $k$-CNF formula or $k$-SAT formula is a CNF formula with $k$ literals per clause).
A Horn-SAT formula is a Boolean formula  in CNF where each clause has at most
one positive literal (each such clause is called a Horn clause).
Equivalently, a Horn clause can be written as an implication $ ((v_1
\AND v_2 \ldots \AND v_r) \rightarrow u)$ where $u, v_1, v_2, \ldots, v_r$
are pure literals and $r \geq 0$.  A dual-Horn-SAT formula is a
CNF formula where each clause has at most one negative literal.
An Affine-SAT formula is a CNF formula in which each clause is an exclusive-or ($\oplus$)
of its literals or a negation of the exclusive-or of its literals (such a clause is satisfied
exactly when an odd number of  the literals are set to $1$). Equivalently, each clause 
of an Affine-SAT formula can be written as a linear equation over 
the finite field $\{0,1\}$ of $2$ elements.

An \emph{assignment} is a function $X: V \rightarrow \{0,1\}$ that
assigns a truth value (true or false) to each variable in $V$.  Given
such an assignment of truth values to $V$, any Boolean formula $\phi$
defined over $V$ also inherits a truth value (we denote this by
$\phi(X)$), by applying the rules of  Boolean logic. A
\emph{model} is an assignment $X$ such that $\phi(X)$ is true.
 We will often treat an assignment $X$ as an $n$-bit vector where the
$i$-th bit, denoted by $X(i)$, $1 \leq i \leq n$,  is the truth value of the variable $v_i$. With a
slight abuse of notation, we let $X(l)$ denote the value of the
literal $l$ under the assignment $X$.

A $0$-valid-SAT (resp. $1$-valid-SAT) formula is one which is satisfied by an assignment with every variable
set to $0$ (resp. $1$).

The propositional satisfiability problem 
is defined as follows:

\begin{problem}[SAT]
\textbf{Instance:}  A Boolean formula $\phi$. \\
\noindent \textbf{Question:} Does $\phi$ have a model~?
\end{problem}

SAT is the canonical example of an \emph{NP}-complete decision problem
(for definitions of the complexity class \emph{NP} and completeness,
see \citeR<e.g.,>{GJ79,papa}).  Many computational difficult
problems in artificial intelligence have \emph{SAT} encodings (for
example, in planning \cite{kautz92}) and so finding 
heuristic algorithms for solving \emph{SAT} is an important research
area in artificial intelligence.  Polynomial time algorithms are known
for SAT when the input instance is either Horn-SAT, dual-Horn-SAT,
$2$-SAT, Affine-SAT, $0$-valid-SAT or $1$-valid-SAT.
\citeA{Sch78} proved that these are the only cases when SAT is
solvable in polynomial time, every other case being NP-complete
(Schaefer's theorem applies to a more general situation called
``generalized satisfiability'' where the truth value of each clause is
determined by a set of constraints specified as a relation).

\medskip

We now introduce the concept of fault-tolerant models. Given an
$n$-bit assignment $X$, the operation $\delta_i$ flips the $i$-th bit
of $X$ (from a $0$ to a $1$, or vice versa). The operation produces a
new assignment which we denote by $\delta_i(X)$. Similarly, if we flip
two distinct bits (say bits $i$ and $j$), we write the new assignment
as $\delta_{ij}(X)$ and more generally, $\delta_S(X)$ represents $X$
with the bits in $S$ flipped (where $S$ is some subset of the coordinates $\{1, 2,
\ldots, n\}$).

\begin{definition} \label{def2} \label{def:1} A $\delta$-model of a
  Boolean formula $\phi$ is a model $X$ of $\phi$ such that for all
  $i$, $1 \leq i \leq n$, either
  \begin{itemize}
  \item[(i)] the assignment $\delta_i(X)$ is a model or
  \item[(ii)] there is some other bit $j$, where $1 \leq j \leq n$ and
    $i \not= j$, such that $\delta_{ij}(X)$ is a model.
  \end{itemize}
\end{definition}

In other words, a $\delta$-model is a model such that if any bit is
flipped (we call this a \emph{break}), at most one other bit flip is
required to produce a new model.  The second bit flip is called a
\emph{repair}.

\begin{example}  \label{ex:1}
Let $H(n,k)$ be a Boolean formula defined over $n$ variables $v_1, v_2, \ldots, v_n$,  
  whose models are $n$-bit assignments with exactly $k$ bits set to $1$.
  For example:

  \[    H(n,1)  = \left(\bigvee_{i=1}^n v_i\right)  \AND \bigwedge_{i=1}^n  \left\{ v_i \rightarrow \left(\bigwedge_{\substack{ j=1 \\ j \not= i}}^{n} \neg v_j \right) \right\} \]

  The first clause specifies that at least one bit of a model is $1$
  and each successive clause specifies that if the $i$-th bit is $1$,
  then every other bit is set to $0$ where $1 \leq i \leq n$. Each
  model of $H(n,1)$ is a $\delta$-model: any break to a $0$-bit  has a
  unique repair (the bit set to $1$) and a break to the $1$-bit 
  has $(n-1)$ possible repairs (any one of the $0$-bits).
\end{example}

The following decision problem can be interpreted as the fault-tolerant analogue of SAT:

\begin{problem}[$\delta$-SAT]
\textbf{Instance:}  A Boolean formula $\phi$. \\
\noindent \textbf{Question:} Does $\phi$ have a $\delta$-model~?
\end{problem}

The problem $\delta$-SAT and its variants (when we restrict the form of the input Boolean formula) is the focus of this paper.

\medskip

We now extend our notion of single repairability to repairability of a sequence of breaks to a model.

\begin{definition}\label{genmodels}
  A $\de(r,s)$-model of a Boolean formula $\phi$ 
  is  a model of $\phi$ such that  for every choice  of at most
  $r$ bit flips (the  ``break'' set) of the model, there is a disjoint set  of at most $s$ bits (the ``repair'' set) that
  can be flipped to obtain another model  of $\phi$.
\end{definition}

 \begin{remark}

\begin{itemize}
\item[(i)] We view $r$ and $s$ as fixed constants unless otherwise
  mentioned. To avoid redundancies, we have required that the repair
  set is disjoint from the break set.  Since we require ``at most $s$
  bits'' for repair, we also allow for the case when no repair or
  fewer than $s$ repairs are needed.
\item[(ii)] Under this definition, $\delta(1,1)$-models are $\delta$-models and we continue to refer to them as $\delta$-models  for notational simplicity.
\item[(iii)] Similar to the definition of $\delta$-SAT, we can define a decision problem $\delta(r,s)$-SAT which asks whether an input Boolean 
formula has a $\delta(r,s)$-model.
\end{itemize}

\end{remark}

\begin{example}  \label{ex:2}
Each model of $H(n,k)$ (see, e.g., Example~\ref{ex:1})
  is also a $\delta(k,k)$ model when $k \leq n/2$. 
\end{example} 

\medskip

\noindent\textbf{Assumptions:} 
In all our discussions, we will assume that
every variable of an input  Boolean formula appears in both positive and negative literals
and that an input Boolean formula is in clausal form with no variable appearing more than once in a clause
(i.e., there is no clause of the form $v_1 \OR \neg v_1 \OR v_2$). We also assume that
in any instance of $\delta$-SAT (or its variants), 
there is no clause which consists of a single literal, since in that case the input formula cannot have
a $\delta$-model.

\medskip

Consider a $\delta$-model $X$  of a Boolean formula and suppose that
$Y$ is a model which repairs some break to $X$. Our definition (Definition~\ref{def:1}) of
$\delta$-models does not require that $Y$ itself is a $\delta$-model. If
we enforce that every break to $X$ is repaired by some $\delta$-model,
then not only is $X$ tolerant to a single break, but so is the repair. We thus can
define a \emph{degree} of fault tolerance. In this setting, models
will be fault tolerant of degree $0$. Then, $\delta$-models will be
fault-tolerant of degree $1$. More generally, degree-$k$
fault-tolerant models (which we call $\delta^k$-models) consist of
$\delta^{k-1}$ fault-tolerant models such that every break is repaired
by a $\delta^{k-1}$ model.  We give the formal definition below.

\begin{definition}
 Let $\phi$ be a Boolean formula. We define $\delta^k(r,s)$-models inductively: 
$\delta^0(r,s)$-models are models of $\phi$. Then for $k \geq 1$, $\delta^{k}(r,s)$-models of $\phi$ 
 are $\delta^{k-1}(r,s)$-models $X$ of $\phi$ such that for every break of at most $r$ coordinates of $X$,
there is a disjoint set of at most $s$ coordinates of $X$ that can be flipped to 
get a $\delta^{k-1}(r,s)$-model of $\phi$.
\end{definition}

We  define the corresponding decision problem $\delta^k(r,s)$-SAT, which
asks whether an input Boolean formula has a $\delta^{k}(r,s)$-model.
Observe also that by definition a $\delta^k(r,s)$-model is a
$\delta^{i}(r,s)$-model for all $i, 0 \leq i \leq k-1$.

\medskip

\begin{example} \label{ex:3}
Let $n \geq 6$ be even and let $\phi$ be the Boolean formula:
\[ (v_1=v_2) \AND (v_3 = v_4) \AND \cdots (v_{n-1}=v_n) \AND (\bigvee_{k=0}^4 H(n,k)) \]
Then the models of $\phi$ are vectors with either $0$, $2$ or $4$ variables set to $1$. The variables in $\{v_{2i-1}, v_{2i}\}$
have to have the same truth value (and this forces breaks to have unique repairs).

We claim that $X=(0,0, \ldots, 0)$ is a $\delta^2(1,1)$-model of
$\phi$.  Any break (without loss of generality, assume it is to
coordinate $1$) is repaired by a flip to coordinate $2$ (and vice
versa).  The new vector $(1,1,0,0, \ldots, 0)$ is itself a
$\delta$-model.  A break to some other coordinate (say, bit $3$) has a
unique repair (bit $4$) to give a model $(1,1,1,1,0, \ldots, 0)$ with
$4$ $1$'s. This model is no longer repairable, since any model has to
have at most $4$ $1$'s, so a break to any coordinate with a $0$ (e.g., to bit $5$) has no repair.
\end{example}

\bigskip

Let $n \geq 2$ be even. Consider the formula
\begin{eqnarray*}  
(v_1 = v_2) \AND (v_3 = v_4) \cdots \AND (v_{n-1}=v_n) 
\end{eqnarray*}
which has $2^{n/2}$ models. Observe that each model is a $\delta$-model.  So these models are $\delta^k(1,1)$-models for every integer $k \geq 0$. We call these models $\delta^*(1,1)$-models
(as usual, when  $r=1$ and $s=1$, we denote $\delta^*(r,s)$-models as $\delta^*$-models for simplicity).

\begin{definition}\label{def3}
Let $\phi$ be a Boolean formula defined over $n$ Boolean variables. Then a model of $\phi$ which is
  a $\delta^k(r,s)$ model for each $k \geq 0$ is called a $\delta^*(r,s)$-model.
\end{definition}

Observe that the set of all $\delta^*$-models of $\phi$ form a set $M$ of models which satisfies the following properties:
\begin{itemize}
\item[(i)] Each vector in $M$ is a $\delta$-model, i.e., a break to a bit needs at most $1$ repair.
\item[(ii)] When any bit of a vector in $M$ is broken, there is some  repair (if such a repair is needed) such that the new vector is also a member of $M$.
\end{itemize}

We call such sets of $\delta^*$-models \emph{stable sets} of $\phi$.
These stable sets have been studied in a combinatorial setting by \citeA{lr05}.

\begin{remark}
  The existence of families of models which satisfy conditions (i) and
  (ii) above may be used to give an alternate definition of
  $\delta^*$-models which is perhaps more natural.  However, the
  notion of degrees of repairability and that $\delta^*$-models appear
  as the limit of these degrees, is not apparent from this definition,
  hence we use the formulation leading to Definition~\ref{def3}.
\end{remark}

The corresponding decision problem, named $\delta^*$-SAT,
asks whether an input Boolean formula has a $\delta^*$-model. Note
that a ``yes'' answer to this question implies the existence of not one but a 
\emph{family} of such models, in particular, a set $M$ as above.

\medskip

\noindent\emph{Complexity Classes}: We refer to \citeA{papa}
for definitions of basic complexity classes like P and NP.
A language  $L$  is said to be in NEXP if there is a non-deterministic Turing machine (NDTM)
that decides $L$  in exponential time (exponential in the length of the input).
A language $L$ is said to be NP-hard if there is a polynomial time reduction
from SAT to $L$. A language is NP-complete if it is in NP and is  NP-hard.
The complexity class NL (non-deterministic log space), which is contained in P,  consists of languages that are accepted
by non-deterministic Turing machines using space logarithmic in the size of its input. 
The complexity classes $\Sigma_k^P$ are defined as
follows: $\Sigma^P_1$ is NP, $\Sigma^P_k$ for $k \geq 2$ is the set of
languages accepted by a NDTM that has access to an oracle TM for
$\Sigma^P_{k-1}$.

\section{Complexity of Finding $\delta$-models} \label{sec:compl-find-delta}

In this section, we study the computational complexity of finding
$\delta$-models for general Boolean formulas.

\begin{theorem}\cite{gins} \label{NPC}
The decision problem $\de(r,s)$-SAT is NP-complete.
\end{theorem}

\begin{remark} The proof technique used in \citeA{gins} to prove Theorem~\ref{NPC} is used
to prove other NP-hardness results in this paper, e.g., in Theorem~\ref{superstar} and Theorem~\ref{thr:1}.
\end{remark}

\begin{theorem} \label{superstar}
The decision problem $\delta^*$-SAT is in NEXP and is NP-hard.
\end{theorem}

\begin{proof} 

  Since an NDTM can guess a stable set of models (which could be of
  exponential size) and check that it  satisfies the required
  conditions for stability in exponential time, $\delta^*$-SAT is in
  NEXP.

  We reduce SAT to $\delta^*$-SAT using the same reduction used in the
  proof of Theorem~\ref{NPC} in \citeA{gins}: given an instance $\phi$
  of SAT, a Boolean formula $\phi$ over $n$ variables $v_1, v_2,
  \ldots, v_n$, we construct an instance of $\delta^*$-SAT: the
  formula $\phi'=\phi \OR v_{n+1}$ with $v_{n+1}$ being a new variable
  (to put $\phi'$ in CNF form, we add the variable $v_{n+1}$ to each
  clause in the CNF formula $\phi$).

  Suppose $\phi$ has a model $X$.  We show that $\phi'$ has a
  $\delta^*$-model by showing that it has a stable set of models $M$.
  Extend $X$ to a model $Y$ of $\phi'$ by setting $v_{n+1}=0$.  Let
  $X_i = \delta_i(X)$ for $1 \leq i \leq n$. Extend each assignment
  $X_i$ to a model $Y_i$ of $\phi'$ by setting $v_{n+1}=1$. Then let
  \[ M= \{Y, Y_1, Y_2, \ldots, Y_n\}. \] We now show that $M$ is a
  stable set. Suppose some bit $j \not= i$, where $1 \leq j \le n$ of $Y_i$ is
  broken, then repair by flipping the $i$-th bit (in which case, we
  get the repaired vector $Y_j \in M$). If the $i$-th bit of $Y_i$ is
  broken, the repair is the $(n+1)$-th bit (and vice versa), in which
  case the repaired vector is $Y$. If instead the $i$-th bit of $Y$ is broken,
  where $1 \leq i \leq n$, then the repair is the $(n+1)$-th bit (we obtain 
  $Y_i$ as the repaired vector in this case). If the $(n+1)$-th bit of $Y$ is broken,
  we can repair by flipping any of the first $n$ bits.  Hence $M$
  is a stable set of models and so $\phi'$ has a $\delta^*$-model (in
  fact, we have exhibited $n+1$ such models).

  Now we show that if $\phi'$ has a $\delta^*$-model, then $\phi$ has
  a model.  If $\phi'$ has a $\delta^*$-model, it must have a
  $\delta^*$-model with the $(n+1)$-th coordinate set to $0$. Then the
  restriction of this assignment to $v_1, v_2, \ldots, v_n$ has to be
  a model of $\phi$.  This completes the reduction from SAT.
\end{proof}

\begin{remark} Note that while every $\de^*$-model is a $\de^k$-model for each $k \geq 1$,
the NP-hardness of $\de^*$-SAT (Theorem~\ref{superstar}) does not imply the NP-hardness of $\de^k$-SAT (Theorem~\ref{thr:2} below).
The reduction used in Theorem~\ref{superstar}  can however be adapted to prove Theorem~\ref{thr:2}.
\end{remark}

\begin{theorem}\label{thr:2}
$\de^k$-SAT is NP-complete, where $k \geq 0$.
\end{theorem}

\begin{proof} When $k=0$, this is Cook's Theorem~\cite{GJ79}, so assume that $k \geq 1$.
  First observe that $\de^k$-SAT is in NP.  This is because an NDTM can
  guess an assignment $X$ and check that it is a $\de^k(1,1)$-model:
  to check whether $X$ is a $\de^k(1,1)$-model, it suffices to
  consider all possible $n^k$ break sets, and check that a repair
  exists for each break applied in sequence from the break set.  Since
  $k$ is fixed, this can be done in polynomial time.

  To prove that $\de^k$-SAT is NP-hard, we use, once again, the proof
  technique used in~\citeA{gins} to prove Theorem~\ref{NPC}. Given an
  instance $\phi$ of SAT, defined on $n$ variables
  $v_1, v_2, \ldots, v_n$, we construct $\phi'= \phi \OR v_{n+1}$
  (and modify $\phi'$ to a CNF formula), where $v_{n+1}$ is a newly
  introduced variable. The argument used in Theorem~\ref{superstar} can now
  be used to prove that $\phi$ is satisfiable iff $\phi'$ has a $\de^k$-model.
  In particular, we  construct a stable set of models $M$ for $\phi'$ from a single model of $\phi$.
  Since a $\delta^*$-model is a $\de^k$-model, this proves that if $\phi$ is satisfiable, then $\phi'$ has a $\delta^*$-model.
  The other direction also follows: if $\phi'$ has a $\delta^k$-model then it has a model with $v_{n+1}$ set to $0$. The restriction
  of that model to $v_1, \ldots, v_n$ is a model of $\phi$.
\end{proof}

\section{Finding $\delta$-models for Restricted Boolean Formulas}\label{section-restrict}
  
In this section, we consider the complexity of $\de(r,s)$-SAT for
restricted classes of SAT formulas which are known to have
polynomial-time algorithms for satisfiability: $2$-SAT, Horn-SAT,
dual-Horn-SAT, $0$-valid SAT, $1$-valid SAT and Affine-SAT.  We
observe that these problems have different complexity of testing fault
tolerance.  For example, $2$-SAT and Affine-SAT have polynomial time
tests for the existence of $\delta$-models (see Section~\ref{subsection-2sat} and
\ref{sec:finding-de-models-affine}) whereas the same problem is NP-complete for
Horn-SAT (Section~\ref{subsection-horn}).

\subsection{Finding $\de$-models for $2$-SAT} \label{subsection-2sat}

We now prove that finding $\delta$-models for $2$-SAT formulas is in
polynomial time. We give two independent proofs: the first proof (Section~\ref{subsubsection-2sat1})
exploits the structure of the formula and the second proof (suggested
by a referee) uses CSP (constraint satisfaction problem) techniques (Section~\ref{subsubsection-2sat2}).
In contrast, we show that finding $\de(1,s)$-models for $2$-SAT formulas is NP-complete for $s \geq 2$
(Section~\ref{subsubsection-2sat3}). However, we also show that finding $\de^*$-models for $2$-SAT formulas is in polynomial time (Section~\ref{subsubsection-2sat4}).

\bigskip

\subsubsection{Polynomial time algorithm for $\de(1,1)$-$2$-SAT} \label{subsubsection-2sat1}

\noindent\textbf{Notation:}
Let $\phi$ be an instance of $2$-SAT. Following the notation
in~\citeA{papa}, we define the directed graph $G(\phi)=(V,E)$ as
follows: the vertices of the graph are the literals of $\phi$ and for
each clause $l_i \rightarrow l_j$ (where $l_i, l_j$ are literals), there are
two directed edges $(l_i, l_j)$ and $(\NOT l_j,\NOT l_i)$ in 
$E$. A \emph{path} in $G(\phi)$ is an ordered sequence of vertices
$(l_1, l_2, \ldots, l_r)$ where $(l_i, l_{i+1}) \in E$ for $1 \leq i
\leq r-1$.  We define a
\emph{simple path} in $G(\phi)$ to be a path $(l_1, l_2, \ldots, l_r)$
where the literals $l_i$ involve distinct variables, i.e., $l_i \not=
l_j$ and $l_i \not= \neg l_j$ for all $i \not= j$, where  $1 \leq i, j \leq r$.
A simple cycle of $G(\phi)$ is a simple path where we allow the start
and end vertices to be identical. A \emph{source vertex} (resp. a
\emph{sink vertex}) in $G(\phi)$ is a vertex with in-degree
(resp. out-degree) $0$. A vertex $l$ in $G(\phi)$ is said to be a
$k$-ancestor (resp. $k$-descendant) if there exists a simple path
$(l,l_1, l_2, \ldots, l_k)$ (resp. $(l_1, l_2, \ldots, l_k, l)$) of
length $k$ in $G(\phi)$.

\bigskip

The following well-known lemma provides a necessary and sufficient
condition for a $2$-SAT formula to be satisfiable.

\begin{lemma}\label{sat2}\cite{papa}
A $2$-SAT formula   $\phi$ is unsatisfiable iff there is a variable $x$ appearing in $\phi$  such that there
is a path from $x$ to $\NOT x$ and a path from $\NOT x$ to $x$ in $G(\phi)$.
\end{lemma}

If $\phi$ has a $\delta$-model, then $G(\phi)$ has further
restrictions.

\begin{lemma}\label{rest}
 If a $2$-SAT formula $\phi$ has a $\delta$-model, then there is no path from $l$ to $\NOT l$ for any vertex $l$ in $G(\phi)$.
\end{lemma}

\begin{proof} If there was a path from $l$ to $\NOT l$ in $G(\phi)$, then any satisfying assignment
has to set $l$ to false. If we now flip the value of the literal $l$ (by flipping the associated
variable), we cannot repair to get a model of $\phi$. 
\end{proof}

\begin{remark} Lemma~\ref{rest} establishes a necessary condition for
  a satisfiable $2$-SAT formula to have a $\delta$-model. Unlike   Lemma~\ref{sat2},
  this condition is not sufficient: consider, for example,  the $2$-SAT formula
  (which also illustrates many of the constraints that have to be satisfied if a $\delta$-model exists):
\[ (v_1 \rightarrow v_2) \AND (v_2 \rightarrow v_3) \AND (v_3 \rightarrow v_4) \AND (v_4 \rightarrow v_5). \]
Any $\delta$-model of this formula has to set $v_1$ to false (otherwise every variable has to be set to true and a break to $v_5$
requires more than one repair). Similarly $v_5$ has to set to $1$, $v_2$ to $0$ and $v_4$ to $1$. No choice of $v_3$
will allow a single repair to a break to both $v_1$ or $v_5$. This formula thus does not have a $\delta$-model, yet
it satisfies the necessary condition of Lemma~\ref{rest}.
\end{remark}

We now establish a necessary and sufficient condition for a model of $2$-SAT formula $\phi$ to be a  $\delta$-model.

\begin{lemma}\label{path}
Let $\phi$ be a satisfiable $2$-SAT formula. Suppose that there is no path from $l$ to $\neg l$ for any vertex $l$ in $G(\phi)$.
Let $X$ be a model of $\phi$. Then $X$ is a $\delta$-model if and only if it satisfies the following conditions:
\begin{enumerate}
\item[(C1)]  Let $\mc{P} = (l_1, l_2, l_3)$ be a simple path in $G(\phi)$ 
of length 2. Then $X(l_1)=0$ and $X(l_3)=1$. 
\item[(C2)]  If $(l_1, l_2)$ and $(l_1, l_3)$ are edges in $G(\phi)$, 
then $X(l_1), X(l_2), X(l_3)$ cannot all be $0$.
\end{enumerate}
\end{lemma}

\begin{proof}  ($\Rightarrow$) Suppose $X$ is a $\delta$-model of $\phi$.
Let $P=(l_1, l_2, l_3)$ be a simple path of length $2$. If $X(l_1)=1$,
then  $X(l_2)=X(l_3)=1$, otherwise $X$ cannot be a model of $\phi$. A break to $l_3$  requires
the values of both $l_1$ and $l_2$ to be flipped so $X$ cannot be a $\delta$-model, a contradiction.
So $X(l_1)=0$. Similar arguments show that $X(l_3)=1$. Condition (C2) holds similarly: if $X(l_1), X(l_2), X(l_3)$
were all false, then a break to $l_1$ would require two repairs (both $l_2$ and $l_3$). Hence one of them has to be set to true.

($\Leftarrow$) Let $X$ be a model of $\phi$ which satisfies conditions
(C1) and (C2). We show that $X$ is actually a $\delta$-model.  Suppose
not; say a break to a variable $v$ is not repairable by
at most one other bit flip.  Assume without loss of generality, that
$X(v)=0$ and so after the break, $v$ is set to $1$. There must be at
least one clause of the form $v \rightarrow l$ where $l$ is a literal,
with $X(l)=0$, otherwise the break does not need a response.
 If there is more
than $1$ such clause, say clauses $v \rightarrow l$ and $v \rightarrow
l'$ with $X(l)=X(l')=0$, then $X$ violates condition (C2),
contradicting the hypothesis. So there is exactly one clause of the
form $v \rightarrow l$ with $X(l)=0$ and moreover, it must be the case
that flipping $l$ does not produce a model of $\phi$ (then one repair
would have sufficed). Now since a
flip of the variable associated with $l$ repairs the clause $v
\rightarrow l$, there must be other clauses that break when $l$ is
repaired. Such a clause must be of the form $l \rightarrow l'$ for
some literal $l'$ with $X(l')=0$. We know that $l'$ cannot be $\neg v$
since then we would have a path between $v$ and $\neg v$ in $G(\phi)$,
which violates the hypothesis. Our assumption that each clause has distinct literals implies that  $l' \not= \neg l$. Hence
$(v,l,l')$ is a simple path such that $X(v)=0$ and $X(l')=0$,
contradicting condition (C1).  Hence $X$ is a $\delta$-model.
\end{proof}

\begin{remark}
\begin{description}
\item[(i)]  If $\phi$ has a $\delta$-model, then it is indeed the case that if
  $(v,u)$ and $(w,u)$ are edges of $G(\phi)$, then $u, v$ and $w$
  cannot all be set to true in such a $\delta$-model (since a break to
  $u$ is not repairable by a single flip).  We do not need to include
  this condition explicitly in Lemma~\ref{path}, because this
  condition happens if and only if $(\NOT u, \NOT v)$ and $(\NOT u,
  \NOT w)$ satisfy condition (C2) in Lemma~\ref{path}.
 \item[(ii)] If $\phi$ has a $\delta$-model, then condition (C1) can be extended to specify the values of literals (vertices)
on any path of length $3$ (the maximum possible length, see  Corollary~\ref{restrict1} below)  as follows: if $(u_1, u_2, u_3, u_4)$ is a simple path, then 
 apply condition (C1) twice to get $X(u_1)=X(u_2)=0$ and $X(u_3)=X(u_4)=1$. Thus we do not include this condition explicitly.
\end{description}
\end{remark}

\bigskip

Lemma~\ref{path} has further consequences for $G(\phi)$:

\medskip

\begin{corollary}\label{restrict1}
 If a $2$-SAT formula $\phi$ has a $\delta$-model, then $G(\phi)$ satisfies the following properties:
\begin{itemize}
\item[(i)] The longest simple path in $G(\phi)$ has length at most $3$.
\item[(ii)] The longest simple cycle in $G(\phi)$ has length at most $2$.
\item[(iii)] A vertex $v$ can take part in at most $1$ simple cycle.
\end{itemize}
\end{corollary}

\begin{proof}
Suppose that there is a simple path $(l_1, l_2, l_3, l_4, l_5)$ of length $4$ in $G(\phi)$.
If $X$ is a $\delta$-model of $\phi$, Lemma~\ref{path} implies that $X(l_3)=1$ when we apply (C2) to  the segment $(l_1, l_2, l_3)$
and $X(l_3)=0$ when we apply (C2) to the segment $(l_3,l_4, l_5)$. Hence such a $\delta$-model cannot exist.
The other conditions follow from similar arguments.
\end{proof}

Pseudo-code for our algorithm is given in Algorithm~(\ref{algo}).
Observe that Algorithm~(\ref{algo}) is a polynomial time reduction
from $\delta$-$2$-SAT to the satisfiability question of a new $2$-SAT
formula $\phi_B$. Proof of correctness follows.

\begin{algorithm} 
\caption{Algorithm for $\delta$-2-SAT} \label{algo}
\begin{algorithmic}[1]
\bigskip
\State \textbf{Input}: $2$-SAT formula $\phi$
\State \textbf{Output}: True if $\phi$ has a $\delta(1,1)$-model, false otherwise

\bigskip

  \If{$\phi$ is not satisfiable} 
   \State return false.
  \EndIf

  \bigskip

  \State \emph{/*\  Check if necessary condition  holds  (Lemma~\ref{rest})\  */}
  \State Construct $ G(\phi)$ 
  \If{there is a path in $G(\phi)$ between $l$ and $\neg l$ for any literal $l$}
   \State  return false.
  \EndIf

  \bigskip
   \State $\phi_B \leftarrow \phi $ 

   \bigskip
    
  \State \emph{/*\ Enforce condition (C1) from Lemma~\ref{path} \ */}
\ForAll{$2$-ancestor vertex $l$ in $G(\phi)$} 
  \State $\phi_B \leftarrow  \phi_B \AND (\neg l) $
   \EndFor

  \bigskip
  
  \State \emph{/* \ Force each source (resp. sink) vertex to value $0$ (resp. $1$) \ */}
  \ForAll{source vertices $l$ in  $G(\phi)$}
  \State $\phi_B \leftarrow \phi_B \AND (\neg l) $
  \EndFor

  \ForAll{sink vertices $l$ in $G(\phi)$}
  \State  $\phi_B \leftarrow \phi_B \AND (l)$
  \EndFor

  \bigskip

  \State \emph{/* \ Enforce condition (C2) from Lemma~\ref{path} \ */}
    \ForAll{$1$-ancestors $l$  in $G(\phi)$}
    \ForAll{ pairs of distinct vertices $l_1, l_2$ } 
     \If{$(l,l_1), (l,l_2)$ are edges in  $G(\phi)$}
       \State $\phi_B \leftarrow\phi_B \AND (l_1 \OR l_2)$
     \EndIf
    \EndFor
  \EndFor

  \bigskip

  \If{$\phi_B$  is satisfiable}
  \State return true
   \Else 
   \State return false
   \EndIf

\end{algorithmic}
\end{algorithm}

We first need to prove the following easy lemma.

\begin{lemma} \label{changedmodel}
 If a $2$-SAT formula $\phi$ has a $\de$-model, then it has a $\de$-model
with each source vertex (respectively, sink vertex) in $G(\phi)$ set to false (resp. true).
\end{lemma}

\begin{proof}
  Modify a $\delta$-model $X$ of $\phi$ by setting each sink vertex to
  $1$ (and hence each source vertex to $0$).  Let the new assignment
  be $X'$. Clearly, $X'$ is still a model of $\phi$ (setting the
  antecedent $p$, or the consequent $q$, to $0$, or $1$ respectively, satisfies every
  implication $p \rightarrow q$). We show that this model satisfies
  condition (C1) and (C2) of Lemma~\ref{path}, thus proving that it is
  a $\delta$-model. If condition (C1) is violated, then there is some
  simple path $(l_1, l_2, l_3)$ in $G(\phi)$ where $X'(l_1)=1$ or
  $X'(l_3)=0$.  If $X'(l_1)=1$, then $X(l_1)=1$ (suppose not and let
  $X(l_1)=0$: since $(l_1,l_2)$ is an edge in $G(\phi)$, $l_1$ is not
  a sink vertex, so its value would not have been changed). Similarly,
  $X'(l_3)=0$ would imply that $X(l_3)=0$.  Thus $X$ would violate
  condition (C1) with respect to the simple path $(l_1, l_2, l_3)$ and
  could not have been a $\delta$-model (a contradiction). Condition
  (C2) similarly holds.
\end{proof}

Algorithm (1) adds literals to the input $2$-SAT formula $\phi$ to
enforce variable assignments that \emph{must} hold if $\phi$ has a
$\delta$-model (see Lines $12$--$15$, $24$--$30$ in the body of Algorithm (1)). Since we are
guaranteed by Lemma~\ref{path} that these conditions are a necessary
and sufficient condition for the existence of a $\delta$-model, the
satisfiability of the resulting Boolean formula would imply that
$\phi$ has a $\delta$-model. To simplify the proof of correctness
(which is now simply Corollary~\ref{cor:1} below), we enforce that
source and label vertices get default values prescribed by
Lemma~\ref{changedmodel}.

\begin{corollary} \label{cor:1}
The formula $\phi_B$ is satisfiable iff $\phi$ has a $\delta$-model. 
\end{corollary}

\begin{proof} Immediate from Lemma~\ref{path} and Lemma~\ref{changedmodel}.
\end{proof}

\medskip

\begin{example}\label{ex:4}
Let $\phi$ be the $2$-SAT formula:
 \begin{gather*}
(v_1 \rightarrow v_2) \AND (v_2 \rightarrow v_3) \\
(v_1 \rightarrow v_4) \AND (v_4 \rightarrow v_3) \\
(v_1 \rightarrow v_5) \AND (v_5 \rightarrow v_3)
\end{gather*}
Then Algorithm~(\ref{algo}) constructs $\phi_B$ where
 \begin{gather*}
\phi_B =   \phi \AND \\
    (\neg v_1) \AND (v_3)  \mbox{\ \ \ \ \ \ (\emph{added by lines 13--16 in Algorithm (1)})} \\
   \AND (v_2 \OR v_4) \AND (v_2 \OR v_5) \AND (v_4 \OR v_5)   \mbox{\ \ \ \ \ (\emph{added by lines 24--31})} \\
  \AND (\neg v_2 \OR \neg v_4) \AND (\neg v_2 \OR  \neg v_5) \AND (\neg v_4 \OR \neg v_5) \mbox{\ \ \ (\emph{added by lines 24-31})}
\end{gather*}
Note that in the construction of $G(\phi)$, $\neg v_3$ is a $2$-ancestor. 
Since  two of the variables $v_2, v_4, v_5$ have to be set to the same value, $\phi_B$ is unsatisfiable.
Hence $\phi$ does not have a $\delta$-model.
\end{example}

\medskip

\begin{theorem} \label{2sat-poly}
 In polynomial time, one can determine if a $2$-SAT formula has a $\delta$-model
and find one if it exists.
\end{theorem}

\begin{proof}
 Satisfiability of a $2$-SAT formula is in $P$ \cite{papa}. Other steps in the procedure
consist of looping over simple paths of length $3$, which can be done in time $O(n^3)$ where
$n$ is the number of variables.  
\end{proof}

\begin{remark}
It is possible to further characterize the space complexity of $\de(1,1)$-2-SAT.
In fact,  $\delta(1,1)$-2-SAT is complete for NL (non-deterministic log space).
To see that $\de(1,1)$-2-SAT is in NL, observe that Algorithm~(\ref{algo}) can be
executed in space logarithmic in the input. Completeness can be established
via a log-space reduction from $2$-SAT.
Since this result is not very relevant in the present context, we leave the details out.
\end{remark}

\subsubsection{An alternative proof of Theorem~\ref{2sat-poly}} \label{subsubsection-2sat2}

An alternative proof of
Theorem~\ref{2sat-poly} was suggested by one of the reviewers. It is
possible to cast any satisfiability problem as a constraint
satisfaction problem (CSP) over binary variables. This transformation,
particularly when the input instance is a $2$-SAT problem, produces a
CSP for which local consistency (consistency of subproblems involving
fewer variables) ensures the presence of a global solution. In this
framework, asserting that a Boolean formula has a $\delta$-model
becomes particularly convenient. 

\medskip

\noindent\textbf{Notation:} Let $\phi$ be a Boolean formula in CNF.  For a subset $S$ of
variables, we let $\phi(S)$ denote the subformula of $\phi$ consisting of
clauses from $\phi$ which only involve variables in $S$.

\medskip

\begin{definition}
  A formula $\phi$ is said to be \emph{$k$-consistent} if for every subset $S$
  of $k-1$ variables, every model of $\phi(S)$ can be extended to a model
  of $\phi(S \cup \{v\})$ for every variable  $v$ (i.e., a larger subformula of
  $\phi$ involving one more variable). A formula is \emph{strong
  $k$-consistent} if it is $i$-consistent for all $i$, $1 \leq i \leq
  k$.
\end{definition}

\begin{remark} The concept of $k$-consistency has other equivalent formulations \cite{jeavons98,dechter92}. Since our goal in this paper is to
  study satisfiability exclusively, we rephrase some of the
  definitions and theorems to apply to our present context. 
\end{remark}

\begin{theorem} \cite{dechter92}\label{dech} Let $\phi$ be a $2$-SAT
  formula. Then the following hold:
\begin{itemize} 
\item[(a)] If $\phi$ is strong $3$-consistent, then $\phi$ is
  satisfiable and for any $2$ element set $S$, $\phi(S)$ is satisfiable.
\item[(b)] In polynomial time (see e.g., \cite{jeavons98}) one can check whether
  $\phi$ is strong $3$-consistent.  If $\phi$ is satisfiable but not
  strong $3$-consistent, then one can add extra clauses (also in $2$-CNF) to
  $\phi$ in polynomial time such that the resulting $2$-SAT formula is
  strong $3$-consistent.
\end{itemize}
\end{theorem}

\begin{remark}
  More generally, given an input Boolean formula $\phi$, one can
  establish $k$-consistency  by adding extra
  constraints that do not change the set of models. This is done by
  iterating over all possible $k$-element subsets of variables and
  solving the subproblem for these variables.  Clauses are added which
  restrict the values of any subset of $k-1$ variables to only those
  values that can be extended to another variable.  If there is a set
  of $k-1$ variables none of whose assignments can be extended, then
  we can conclude $\phi$ is unsatisfiable.  If not, then these extra
  clauses are added to $\phi$ to make it $k$-consistent.  Enforcing
  strong $k$-consistency (for fixed $k$) can be accomplished in
  polynomial time \cite{jeavons98,dechter92}.

  In the special case when $\phi$ is a $2$-SAT formula these extra
  clauses are also binary and so we end up with a strong
  $3$-consistent $2$-SAT formula (which we denote by $\widehat{\phi}$) with exactly the same models (and
  hence, the same set of $\delta$-models).
\end{remark}

\bigskip

\noindent\textbf{Notation:}  For an ordered pair of 
variables $(u, v)$, we let $M_{\phi}(u, v)$ denote the set of models
of $\phi(\{u,v\})$.   

\bigskip

 Theorem~\ref{dech} (b) implies that  we can assume without loss of generality that the input is a strong
$3$-consistent $2$-SAT formula $\phi$.  Theorem~\ref{dech} also implies
that an assignment $X$ is a model of $\phi$ iff $(X(u), X(v)) \in
M_{\phi}(u, v)$ for all pairs $(u, v)$.  Clearly, we can construct all
the sets $M_{\phi}(u, v)$ in polynomial time (there are $\Theta(n^2)$
such variable tuples, where $n$ is the number of variables, and each
set $M_{\phi}(u,v)$ consists of models of a $2$-SAT formula with at
most $2$ variables). With a slight abuse of notation, we denote
$M_{\phi}(-u,v)$ to be the set $\{(\neg \alpha, \beta)|\ (\alpha,
\beta) \in M_{\phi}(u,v) \}$.

Let  $u$ be any
variable of $\phi$. Let $\phi_{u,0}=\phi \AND (\neg u)$ and
$\phi_{u,1}=\phi \AND (u)$.  If either $\phi_{u,0}$ or ${\phi}_{u,1}$ is
unsatisfiable, then it is clear that ${\phi}$ cannot have a
$\delta$-model. Assume then that both ${\phi}_{u,0}$ and ${\phi}_{u,1}$ are satisfiable and 
let $\widehat{\phi_{u,0}}$ and $\widehat{\phi_{u,1}}$ be the corresponding strong $3$-consistent formulas.
Let $N_u$ be the set of variable pairs $(v,w)$ such that
$M_{\widehat{\phi_{u,0}}}(v, w) \cap M_{\widehat{\phi_{u,1}}}(v,w) = \emptyset$.

\begin{lemma} \label{pair}
  Suppose $N_u \not= \emptyset$ for some variable $u$.
  If $\phi$ has a $\delta$-model, then there is some variable $v$, where $v \not= u$, such that $v$ belongs to every pair in $N_u$.
\end{lemma}

\begin{proof} If we flip the value of $u$ in a $\delta$-model of $\phi$, we can repair by flipping
at most one other variable and we are forced to flip one variable from each pair in $N_u$. This means that this repair variable is in
every pair of $N_u$.
\end{proof}

\begin{algorithm}
\caption{Algorithm for $\delta(1,1)$-$2$-SAT}
\begin{algorithmic}[1]
 \bigskip
  \State \textbf{Input:} A strong $3$-consistent $2$-SAT formula $\phi$
  \State \textbf{Output:} True if $\phi$ has a $\delta(1,1)$-model, false otherwise
  \bigskip

  \For{every variable $u$}
   \If{$\phi_{u,0}$ or $\phi_{u,1}$ is unsatisfiable}
   \State Output false.
  \EndIf
   \State Find sets $M_{\widehat{\phi_{u,0}}}(v, w)$ and
  $M_{\widehat{\phi_{u,1}}}(v, w)$ for variables $v, w$.  \State Compute $N=$
  set of pairs $(v, w)$ such that $$M_{\widehat{\phi_{u,0}}}(v, w) \cap
  M_{\widehat{\phi_{u,1}}}(v, w) = \emptyset.$$ 
  
  \State If the pairs in $N$ do
  not have a common member, then output false.  
  
  \If{$N \not=
    \emptyset$} \State For the common member $v$, \If{there is a
    variable $w$ such that $M_{\widehat{\phi_{u,0}}}(-v,w) \cap
    M_{\widehat{\phi_{u,1}}}(v, w) = \emptyset$} 
    \State set $\phi=\phi \AND (\neg u)$
  \EndIf
  \If{there is a variable $w$ such that $M_{\widehat{\phi_{u,0}}}(v,w) \cap M_{\widehat{\phi_{u,1}}}(-v,w) = \emptyset$}
  \State set $\phi=\phi \AND u$
  \EndIf
  \EndIf
  \State Check if $\phi$ is satisfiable, if not output false.
  \State If $\phi$ is satisfiable, add extra clauses to $\phi$ to make it $3$-consistent.
  \EndFor
\State{Output true}

\end{algorithmic}
\end{algorithm}

Lemma~\ref{pair} implies that we may assume that the pairs in $N_u$
have a common member. We can similarly show:

\begin{lemma} \label{force} Suppose that $v$ is a variable that
  appears in every pair in $N_u$. Then the following hold:  
  \begin{itemize}
  \item[(i)] If there exists a $w$ 
  such that    
    \[ M_{\widehat{\phi_{u,0}}}(-v,w) \cap M_{\widehat{\phi_{u,1}}}(v,w) = \emptyset  \]
   then any $\delta$-model $X$ of $\phi$ has to set $X(u)=0$.
   \item[(ii)] If there exists a $w$ such that 
  \[ M_{\widehat{\phi_{u,0}}}(v,w) \cap M_{\widehat{\phi_{u,1}}}(-v,w) =  \emptyset, \]
   then any $\delta$-model $X$ of $\phi$ has to set $X(u)=1$.
  \end{itemize}
\end{lemma}

Thus either of the two conditions in Lemma~\ref{force} force the value
of the variable $u$ in any $\delta$-model of $\phi$. Together
Lemmas~\ref{pair} and~\ref{force} enable us to set the values of the
variables that are forced (cf. Lemma~\ref{path}). If after setting the
values of these variables, we derive a contradiction then $\phi$
cannot have a $\delta$-model.

Algorithm (2) provides the detailed description of the algorithm.

\begin{theorem}
Algorithm~(2) decides $\delta(1,1)$-2-SAT in polynomial time.
\end{theorem} 

\begin{proof}
Enforcing $3$-consistency is in polynomial time \cite{dechter92}. The outer loop in Line 3 executes $n$ times where $n$ is the number
of variables. Within the body of the loop, calls are made to enforce satisfiability and $3$-consistency, along with calls to construct $N_u$ for 
the variable $u$ under consideration. Each step takes polynomial time, hence the claim follows.
\end{proof}

\begin{remark} While Algorithm (2) solves the yes/no problem of testing whether an input $2$-SAT formula
has a $\delta$-model, it is a simple matter to modify the algorithm so that it outputs a $\delta$-model
if such a model exists. The forced variable assignments along with  any satisfying assignment of the remaining
$2$-SAT formula is a $\delta$-model of the input formula.
\end{remark}

\subsubsection{Complexity of $\delta(1,s)$-$2$-SAT for $s \geq 2$}\label{subsubsection-2sat3}

\begin{theorem}\label{2sat-(1,b)} 
The problem $\delta(1,s)$-2-SAT  is NP-complete for all $s>1$.
\end{theorem}

\begin{proof}  
Clearly this problem is in  NP: an NDTM can guess such an assignment
and check that it is a model and that for every break, there are at most $s$ other bits
that can be flipped to get a model (since $s$ is fixed a priori, this leads to at most $O(n^s)$
 possible repair sets, a polynomial number of choices).

 We prove NP-completeness via a  reduction from $(s+1)$-SAT.
Let $$T=C_1 \AND C_2 \ldots \AND C_m$$ be an
instance of $(s+1)$-SAT where each clause $C_i$ is a disjunction of
$s+1$ literals: $$v_{i,1} \OR v_{i,2} \ldots \OR v_{i, s+1}.$$  

We construct an instance $T'$ of $\delta(1,s)$-2-SAT as follows:
for each clause $C_i$ in $T$, we construct an appropriate $2$-SAT formula $C_i'$.
Our resulting instance of $\delta(1,s)$-2-SAT is a conjunction of these $2$-SAT formulas. Thus,
$$T' = \bigwedge_{1 \leq i \leq m} C_i' $$
where $C_i'$ is a $2$-SAT formula defined for each clause $C_i$ as follows:
\begin{equation} \label{2satgadget}
\begin{split}
 C_i'  = & \bigwedge_{1 \leq j \leq (s+1)}   (z_i \Rightarrow v_{i,j}) \\
         & \bigwedge_{1 \leq j \leq (s+1)}  (v_{i,j} \Rightarrow \alpha_{i,j,1}) \\
         & \bigwedge_{1 \leq j \leq s+1}\  \bigwedge_{1 \leq k \leq (s-1)}  (\alpha_{i,j,k} \Rightarrow \alpha_{i,j,k+1})\\
\end{split}
\end{equation}
where we have introduced $1 + s(s+1)$ new variables:  $z_i$ and $\alpha_{i, j, k}$ for
$ 1 \leq j \leq s+1, 1 \leq k\leq s$ to define the gadget $C_i'$. 
The gadget $C_i'$ is best understood via Figure~(\ref{2sat}).

\begin{figure}

\begin{center}
  \scalebox{0.40}[0.40]{ \includegraphics{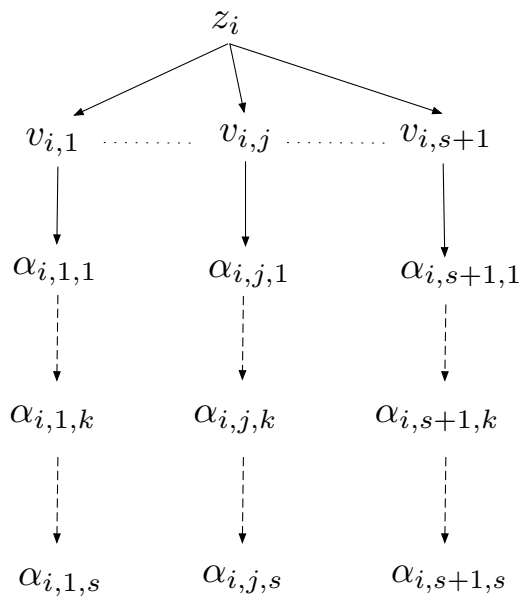}}

\end{center}
\caption{Gadget for $2$-SAT}
\label{2sat}
\end{figure}

 Let $T$ have  a model $X$. Extend that to a model of $T'$ by setting 
$z_i=0$ for all $1 \leq i \leq m$ and $\alpha_{i, j, k}=1$ for all 
$1 \leq i \leq m, 1 \leq j \leq s+1, 1 \leq k\leq s$. We claim that this is
a $\de(1,s)$-model of $T'$.  Suppose we flip the variable corresponding to literal $l$. Now we do a case analysis
of how many repairs are needed:
\begin{itemize}
\item  {[$l=z_i$]} Since $v_{i,1} \OR v_{i,2} \ldots \OR v_{i, s+1}$ is set true by the model $X$,   we need to flip at most $s$ false literals in 
$\{ v_{i,1}, \ldots,  v_{i, s+1}\}$. Observe that no more repairs are necessary.

\item {[$l=\alpha_{i,j,k}$]} Need to flip $\alpha_{i,k,k'}$ where $1 \leq k' < k$ and we might
need to flip the variable corresponding to $v_{i,j}$ if $v_{i,j}$  was set to true by $X$. This repair does not affect the truth value of other clauses of $T'$. 
Hence we flip at most  $s$ variables.

\item  {[$l=$ variable occurring in $T$]}  This will flip the value of all literals involving $l$. Because we set every $a_{i, j, k}=1$ and $z_i=0$, no repairs are needed in $T'$, as each implication (clause) of $T'$ still remains true.
\end{itemize}

 Now suppose $T'$ has a $\de(1,s)$-model. We show that $T$ has a model. Note that in such a
model $z_i=0$ for all $i$ (otherwise if $z_i=1$, then $v_{i,j}= \alpha_{i, j, k}=1$ and we will need more than $s$
repairs when we flip the value of $\alpha_{i, 1, s}$).  Now all literals $\{v_{i,1},
v_{i,2} , \ldots, v_{i, s+1} \}$ cannot be set to $0$, since a break to
$z_i$ would again necessitate $s+1$ repairs.  Hence at least one of
the literals in $\{v_{i,1},  v_{i,2}, \ldots, v_{i, s+1}\}$ is set to
$1$. In other words, the clause $C_i$ in $T$ is satisfied. Since
$z_i=0$ for all $i$, $T$ must have a model.
\end{proof}

\bigskip

\subsubsection{Complexity of $\delta^*$-$2$-SAT} \label{subsubsection-2sat4}

In this section, we show that $\delta^*$-$2$-SAT is in polynomial time.

\medskip
 
Let $\phi$ be the input $2$-SAT formula over $n$ variables.  
We construct the graph $G(\phi)$ as described before in Section~\ref{subsubsection-2sat1}.
Since a $\delta^*$-model is by definition also a $\delta$-model, we
must have the same path restrictions set forth by Lemma~\ref{path} and
Lemma~\ref{rest}.  If $\phi$ has a $\delta^*$-model, then
$G(\phi)$ has further restrictions.

\begin{lemma} \label{path*}
 Let $\phi$ be a $2$-SAT formula with a $\delta^*$-model. Then every non-trivial simple path in $G(\phi)$ has length $1$.
\end{lemma}

\begin{proof} Suppose that $(l_1, l_2, l_3)$ is a simple path
  in $G(\phi)$ of length $2$.  Let $X$ be a $\delta^*$-model of $\phi$.
  Because of Lemma~\ref{path}, we know that $X(l_1)=0, X(l_3)=1$ and this
  has to be the case for \emph{all} $\de^*$-models. This means that a
  break to $X(l_1)$ cannot be repaired to get another $\de^*$-model. 
  Hence, $X$ cannot be a $\delta^*$-model, a contradiction.
\end{proof}

\bigskip

\begin{remark}
  Note $G(\phi)$ may have cycles $( l_1 , l_2,  l_1)$, however in that
  situation, Lemma~\ref{path*} implies that $\{l_1, l_2\}$ must form one
  connected component.  Any $\delta^*$-model if it exists assigns the
  same value to $l_1$ and $l_2$ such that the  respective variables
  form a break-repair pair and are independent of the remaining
  variables.  We can thus remove the cycles from consideration. So
  without loss of generality, we assume that $G(\phi)$ has no cycles.
\end{remark}

\bigskip

Let $R$ be the vertices in $G(\phi)$ with in-degree $0$ and $B$ be the
vertices with out-degree $0$.  Since a vertex cannot have positive
in-degree and positive out-degree, this creates a bipartition $R \cup
B$ of the vertices of $G(\phi)$, where $R, B$ are disjoint vertex sets
and all edges in $G(\phi)$ are of the form $( l, l')$ with $l \in R$ and
$l' \in B$.

Note that if $(l, l')$ is an edge in $G(\phi)$, then the out-degree of
$\NOT l$ is $0$: otherwise, there would be a path of length $2$ or a
cycle, both of which we have excluded.  Hence $l \in R$ iff $ \NOT l
\in B$. We also observe that there are no isolated vertices in $G(\phi)$
since every clause is a disjunction of distinct literals.  This gives
a complete graph theoretic characterization of the structure of
$G(\phi)$ when $\phi$ has a $\de^*$-model.

Now let $Y_0$ be an assignment that sets every literal in $R$ false
(0) and (hence sets) every literal in $B$ true (since we have assumed
that every variable appears in both positive and negative literals).

\begin{lemma}
   The assignment $Y_0$ is a $\delta^*$-model.
\end{lemma}

\begin{proof}  We exhibit a stable set $\mc{C}$ of models  of $\phi$
that contains $Y_0$.  Let $Y\!\downharpoonright_B$ (respectively, $Y\!\downharpoonright_R$) denote the restriction of 
an assignment $Y$ onto the literals in $B$ (respectively, $R$).

 Let 
$$\mc{C}=\{Y\, |\  \mbox{\ $Y\!\downharpoonright_B$ contains at most  one literal set false\ } \}.$$ 

Note that if $Y \!\downharpoonright_B$ contains at most one false literal, then
$Y \!\downharpoonright_R$ contains at most one true literal. Clearly $Y_0 \in \mc{C}$.
 
We now show that $\mc{C}$ is a stable set. Let $Y \in \mc{C}$, where
$Y \not= Y_0$.  Suppose that $Y$ sets the literal $l \in R$ to true
and $\neg l$ to false in $B$.  
If the value of the literal $l$ is flipped, then we get
$Y_0$ (a model in $\mc{C}$) and so no repairs are needed.  If a different variable is flipped, then
this creates a new literal $l'$ in $R$ set to true (and $\neg l'$ false in
$B$) in the new assignment.  Then we repair by flipping the value of $l$ from true to false, thereby
allowing only one positive literal in $R$.
Thus any break to $Y$ is repairable by another
model in $\mc{C}$. A break to $Y_0\in \mc{C}$ does not need any
repairs. Hence $\mc{C}$ is a stable set and $Y_0$ is a $\delta^*$-model.
\end{proof}

\begin{theorem}\label{superstar-2sat} 
$\delta^*$-2-SAT $\in P$.
\end{theorem}

\begin{proof}
The graph $G(\phi)$ can be constructed in polynomial time (in time linear in the size of $\phi$).
All conditions needed for the existence of a $\delta$-model can be checked in polynomial time:
using depth-first search, one can check  if the longest simple path of $G(\phi)$
has length $1$ and check whether the subgraph of $G$ without any $2$-cycles is bipartite.
\end{proof}

\subsection{Finding $\de$-models for Horn-SAT and dual Horn-SAT}\label{subsection-horn}

Recall that an instance of Horn-SAT is a Boolean formula in CNF where
each clause contains at most $1$ positive literal. As in $2$-SAT,
there is a polynomial time algorithm to find a model of a Horn formula
(see, e.g., ~\citeA{papa}). However, {\em unlike} the situation in
$2$-SAT, finding $\delta(1,s)$-models for Horn formulas is NP complete for \emph{all} $s \geq 1$.
The proof of this fact can be easily modified to show that the same problem is NP-complete for dual Horn-SAT.

We first prove a technical lemma which will be used in the
NP-completeness proof. Define the  Boolean formula $\phi=~\phi(x, y, \beta_1, \ldots, \beta_{2s})$ over
variables $x, y, \beta_1, \ldots, \beta_{2s}$ as follows:

\begin{equation}  \label{horneq4}
\begin{split}
 \phi(x, y, \beta_1, \ldots, \beta_{2s}) = & \bigwedge_{i=1}^{s-1} (\beta_i \Rightarrow \beta_{i+1}) \\
              &  \wedge\  (\beta_s \Rightarrow x) \AND (\beta_s \Rightarrow y)  \\
              &  \AND\   (x \Rightarrow \beta_{s+1}) \AND (y \Rightarrow \beta_{s+1}) \AND \\
              &  \bigwedge_{i=s+1}^{2s-1} (\beta_i \Rightarrow \beta_{i+1})
\end{split}
\end{equation}

The formula $\phi$ is best visualized as in Figure~(\ref{Horngadget}). Observe that each variable $x$ and $y$ appears
 both as the head and  tail of a chain of implications of length $s$.

   \begin{figure}[h] 
     \caption{Gadget $\phi$}
     \begin{center}
       \scalebox{0.40}[0.40]{\includegraphics{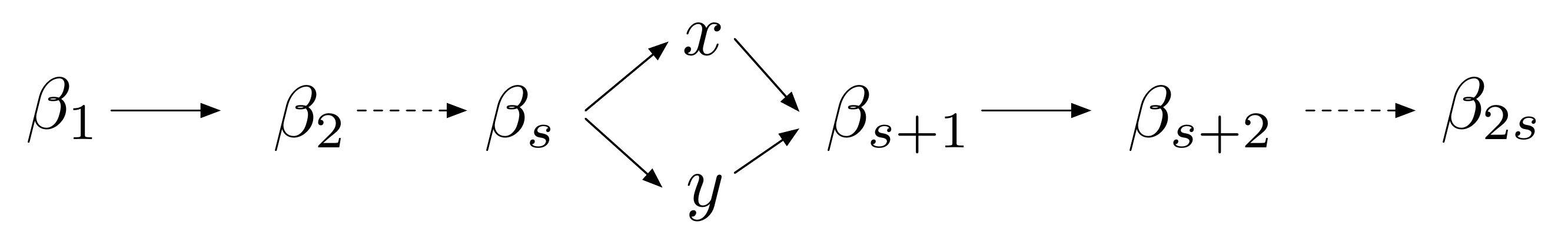}}
            \end{center}
     \label{Horngadget}
   \end{figure}

 The crucial property of this gadget that we use is as follows:
 
\begin{lemma} \label{clm1}
     Let $X$ be a model of $\phi$. Then $X$ is a $\delta(1,s)$-model iff  it satisfies $x \Leftrightarrow \neg y$.
   \end{lemma}

   \begin{proof}
     $(\Leftarrow)$\ Let $X$ be a model of $\phi$. If $x \Leftrightarrow \neg y$ holds for $X$ (i.e., $x$ and $y$ get opposite truth values in $X$), then $X$ has to set all $\beta_i$ with $i \leq s$ to $0$
     (because either $x$ and $y$ is set to $0$)
     and all $\beta_j$ with $j \geq s+1$ to true (because either  $x$ or $y$ is set to $1$). Then a break to $\beta_i$ with $i \leq s$ requires repairs to $\beta_j$ where $i < j \leq s$ and exactly one of $x$ and $y$ (the variable 
     set to $0$). Similarly a break to $\beta_i$ with  $i \geq s+1$ requires repairs to $\beta_j, s+1 \leq j \leq i-1$ and exactly one of $x$ and $y$ (the variable set to $1$). A break to $x$ or $y$ does not need
     any repairs. Since we never need more than $s$ repairs for every break, $X$ is a $\delta(1,s)$-model.
     
     $(\Rightarrow)$ Any $\delta(1,s)$-model $X$ of $\phi$ has to set each $\beta_j$ to $0$ where $1 \leq j \leq s$ and each
     $\beta_j$ to $1$ where $s+1 \leq j \leq 2s $ (otherwise more than $s$ repairs are needed for breaks to these variables).
     If both $X(x)=X(y)=0$, then a break to $\beta_1$ (from a $0$ to $1$) would require repairs to $\beta_2, \beta_3, \ldots,\beta_{s}$
     as well as to \emph{both} $x$ and $y$, for a total of $s+1$ repairs. Hence both $x$ and $y$ cannot be false. Similarly, both $x$ and $y$
     cannot be true because then a break to $\beta_{2s}$ would require $s+1$ repairs. Hence  $X$ satisfies $ x \Leftrightarrow \neg y$.
 \end{proof}

\begin{theorem} \label{Hornb}
$\de(1,s)$-Horn-SAT is NP-complete for $s \geq 1$.
\end{theorem}

\begin{proof} We prove this via a 
  reduction from $3$-SAT. Let $T$ be an
  instance of $3$-SAT, where $T= \bigwedge_{i=1}^m C_i$
  is defined over $n$ variables $x_1, x_2, \ldots, x_n$ and clause $C_i$
  is a disjunction of $3$ distinct literals. Clearly we can assume that
  every variable appears in both positive and negative literals in $T$
  (if not, we may set the pure literal to be true or false appropriately and
  consider the resulting formula as $T$).

  We first apply an intermediate transformation to $T$. We replace any
  positive literal (say $x_j$) in $C_i$ by a  negative literal,
  $\NOT x_j'$, where $x_j'$ is a new variable not occurring in
  $T$. The new clause, which now has no positive literal, is denoted
  by $C_i'$.  Remembering our global assumption that every variable in input
  Boolean formulas appear in both positive and negative literals, we
  see that this transformation will introduce variables $x_j'$ for
  \emph{every} variable $x_j$ in $T$.  To maintain logical
  equivalence, we also need to enforce that $\NOT x_j' \Leftrightarrow
  x_j$ in the new formula: so we add the following clauses: $(\NOT
  x_j' \OR \NOT x_j)$ and $(x_j' \OR x_j)$.  Note that these two
  clauses imply that in any model of this new Boolean formula, $x_j$
  and $x_j'$ cannot have the same truth value.

  Thus we obtain $$T' = \bigwedge_{1 \leq i \leq m} C_i'\ \bigwedge_{1
    \leq i \leq n} \Bigl\{ (\NOT x_i' \OR \NOT x_i) \AND (x_i' \OR
    x_i) \Bigr\} $$
  Note that $T'$ is {\em almost} Horn (since every
  clause $C_i'$ is Horn), the only non-Horn clauses are the clauses of
  the form $(x_i \OR x_i')$.  We have introduced $n$ new variables and
  $2 n$ new clauses, so that $T'$ has $m+2n$ clauses and is defined
  over $2 n$ variables.  Clearly $T'$ is satisfiable iff $T$ is
  satisfiable.

  We now construct an instance $T''$ of $\delta(1,s)$-Horn-SAT from
  $T'$ such that $T''$ has a $\delta(1,s)$-model iff $T'$ is
  satisfiable.  We first introduce $s+1$ new variables $A_1, A_2,
  \ldots, A_{s+1}$.  For each clause $C_i'= \neg v_{i,1} \OR \neg
  v_{i,2}  \OR \neg v_{i,3}$, we construct a formula $\Gamma_{i, 1}$ consisting of a single clause
  (note that at this step, each $v_{i,j}$ is a variable of the form
  $x_k$ or of the form $x_{k}'$ for some $k,\ 1 \leq k \leq n$):
  
  \begin{eqnarray} \label{horneq1}
    \Gamma_{i,1}  =  (\NOT z_i \OR \NOT w_{i,1} \OR \NOT w_{i,2} \OR  \NOT w_{i,3}) 
   \end{eqnarray} 
   where $z_i,\, w_{i,1},\, w_{i,2},\,  w_{i,3}$ are new variables introduced for each clause $C_i'$.
   This step introduces $4$ new variables per clause $C_i'$ for a total of $4 m$ new variables.
   Our next step creates formulas that places restrictions on these new variables and ties them
   in with the variables $v_{i,j}$ in the original clause. We introduce new variables
   $\alpha_{i,j,k}$ for each clause $C_i'$, where $1 \leq j \leq 3, 1 \leq k \leq s-1$, these variables forming the intermediate variables
   in a chain of implications of length $s$ from $v_{i,j}$ to $w_{i,j}$ as below:
   \begin{equation} \label{horneq2}
     \begin{split}
     \Gamma_{i,2}= & (v_{i,1} \Rightarrow \alpha_{i, 1, 1}) \AND (\alpha_{i, 1,1} \Rightarrow \alpha_{i, 1,2}) \cdots \AND (\alpha_{i,1,s-1} \Rightarrow  w_{i,1})  \\
     \AND &  (v_{i,2} \Rightarrow \alpha_{i,2,1} ) \AND (\alpha_{i,2,1} \Rightarrow \alpha_{i,2,2}) \cdots \AND (\alpha_{i,2,s-1} \Rightarrow  w_{i,2})  \\
     \AND & (v_{i,3} \Rightarrow \alpha_{i,3,1} ) \AND (\alpha_{i,3,1} \Rightarrow \alpha_{i,3,2}) \cdots \AND (\alpha_{i,3,s-1} \Rightarrow  w_{i,3}) 
     \end{split}
  \end{equation}
   
   The reader may wish to compare the the gadget $\Gamma_{i,2}$ with a similar gadget $C_i'$ in Equation~(\ref{2satgadget}) and shown in  Figure~(\ref{2sat}) that was used in the   proof of Theorem~\ref{2sat-(1,b)}.

  We also make $z_i$, one of the new variables introduced in  $\Gamma_{i,1}$, appear as the head of a chain of implications of length $s+1$ as shown below
in formula $\Gamma_{i,3}$:
  \begin{eqnarray*} \label{horneq3}
    \Gamma_{i, 3}=  (z_i \Rightarrow A_1) \AND (A_1  \Rightarrow A_2) \ldots \AND (A_{s} \Rightarrow A_{s+1})  \\
  \end{eqnarray*}

  We now define the formula $C_i''$ constructed for each clause $C_i'$, $1 \leq i \leq m$, of $T'$:  
 \[ C_i'' = \Gamma_{i,1} \AND \Gamma_{i, 2} \AND \Gamma_{i, 3} \]

 Note that each $C_i''$ is Horn and has introduced new variables $\alpha_{i,j,k}, w_{i,j}, z_i$  for a total of $ 3 (s-1) + 3 + 1= 3 s +1$ new variables.
 The other new variables $A_i$ are global, i.e, reused in the formulas for $C_i''$ for various $i$.

 For the clauses of the form $(\neg x_i' \OR \neg x_i) \AND (x_i' \OR
 x_i)$ from $T'$, where $1 \leq i \leq n$, we introduce new variables
 $\beta_{i,j}$ for each $i$ where $1 \leq j \leq 2s$ and
 construct the gadget $\phi_i=\phi(x_i, x_i', \beta_{i,1}, \beta_{i,2}, \ldots, \beta_{i,2s})$ defined in Equation~(\ref{horneq4}).

  Our instance of $\de(1,s)$-Horn-SAT is then:

   \[  T'' = \bigwedge_{1 \leq i   \leq m} C_i''\, \wedge \bigwedge_{1 \leq i \leq n} \phi_i \]

   We first show that if $T'$ is satisfiable, then $T''$ has a
   $\de$-model.  Suppose $T'$ had a model $X'$. Extend that to an
   assignment $X''$ of the variables of $T''$ by setting the values of
   the newly introduced variables as follows:

\begin{gather*}
  A_i=1  \mbox {\ for \ }  1 \leq i \leq s+1,\\
  z_i=0 \mbox{\ for \ }  1 \leq i \leq m, \\
  w_{i,j}=1 \mbox{\ for all\  $i$ and $j$, where \ }  1 \leq i \leq m \mbox{\ and\ } 1 \leq j \leq 3, \\
  \alpha_{i, j, k}=1 \mbox{\ for all $i$, $j$ where \ } 1 \leq i \leq m, 1 \leq j \leq 3,  1 \leq k \leq s-1,  \\
  \beta_{i,j}=0 \mbox{\ for all $j$,  $1 \leq j \leq s$ and all $i$, $1 \leq i \leq n$},  \\
  \beta_{i,j} = 1 \mbox{\ for all $j$, $s+1 \leq j \leq 2s $ and all $i$, $1 \leq i \leq n$}.
	\end{gather*}
 
Since $X''$ satisfies each clause in $T''$, it is a model of $T''$.
We now show that $X''$ is actually a $\delta(1,s)$-model. Suppose that some variable $v$ of $T''$ is flipped. We do a case by case analysis of the possible repairs
to this break.
\begin{description}
	\item[ [$v = x_i \mbox{\ or \ } x_i'$]]  No repairs are needed since each implication remains satisfied in $T'$.

	\item[[$v=A_i$ for some $i, 1 \leq i \leq s+1$]] The repairs needed are $A_1,
	A_2, \ldots, A_{i-1}$ (since $z_i=0$ for all $i$, it does not need to be flipped) for $i-1$ ($\leq s$) repairs.

	\item[[$v=\beta_{i,j}$ for some $1 \leq i \leq n, 1 \leq j \leq s$]]  The repairs are all $\beta_{i,k}$ where $j+1 \leq k \leq s$. 
	Since $X'$ is a model of $T'$, exactly one of $x_i$ and $x_i'$ is set to false and we need to flip just that variable.
	This leads to at most $ s-j + 1 \leq s$ repairs.

	\item[[$v=\beta_{i,j}$ for some $1 \leq i \leq n, s+1\leq j \leq 2s$]] The repairs needed are $\beta_{i,k}$ for all $s+1 \leq k < j$ and one of $x_i$ or $x_i'$ (since $X'$ is a model of $T'$ only
	one of $x_i, x_i'$ is set to true in $X'$) for at most $j-s \leq s$ repairs.

	\item[[$v=w_{i,j}$ for some $1 \leq i \leq m, 1 \leq j  \leq 3$]] The repairs needed are $\alpha_{i, j, k}$ for all $1 \leq k \leq s-1$ and $v_{i,j}$ (if $X'(v_{i,j})=1$), for at most $s$ repairs.

	\item[[$v=\alpha_{i, j, k}$ for some $1 \leq i \leq m, 1 \leq j \leq 3, 1 \leq k \leq s-1$]]  The repairs needed are $\alpha_{i, j, k'}$ for $1 \leq k'\leq  k-1$ and $v_{i,j}$ (if $X'(v_{i,j})=1$)  for at most $k \leq s-1$ repairs.

	 \item[[$v=z_i$]] It is this break alone whose repair crucially depends on the satisfiability of $T'$. Note that this break changes $z_i$ from a $0$ to a $1$ and makes the clause  $(\NOT z_i \OR \NOT w_{i,1} \OR \NOT w_{i,2} \OR \NOT w_{i,3})$ false since each $w_{i,j}$ is true in $X''$.  So repairs will have to include one or more of the $w_{i,j}$'s, which consequently might trigger flips to $\alpha_{i,j,k}$ and $v_{i,j}$. The choice of which  $w_{i,j}$ to involve in the repair process  is   indicated by the $v_{i,j}$ set to $0$ by $X'$. Since $X'$ is a model, note also that at least one $v_{i,j}$ is set to $0$. Without loss of generality, assume that $X'(v_{i,1})=0$ then repair a break to $z_i$ by flipping $w_{i,1}, \alpha_{i, 1, j}$ for all $1 \leq j \leq s-1$ for exactly $s$ repairs.   		
	 \end{description}

   Now suppose $T''$ has a $\delta$-model $X''$. We show that $T'$ is
satisfiable.  Specifically, we claim that the restriction of $X''$ to
the variables of $T'$ is a model of $T'$. From Lemma~\ref{clm1}, we
know that $\beta_{i, j}=0$ for all $1 \leq i \leq n, 1 \leq j \leq s$
and $\beta_{i, j}=1$ for all $1 \leq i \leq n,\ s+1 \leq j \leq 2s$ and
also $\neg x_i' \Leftrightarrow x_i$ in $T'$ is satisfied for each $i,\ 1 \leq i \leq n$.   Note that in $T''$,  $w_{i,j}$ is at the end of a chain of implications:
\begin{equation} \beta_{k,1} \rightarrow \beta_{k,2} \rightarrow \cdots \rightarrow \beta_{k,s} \rightarrow v_{i,j} \rightarrow \alpha_{i, j, 1} \rightarrow \cdots \rightarrow \alpha_{i, j, s-1} \rightarrow w_{i,j}  \label{bchain}
\end{equation}
where $v_{i, j}$ is either $x_k$ or $x_k'$ for some $k, 1 \leq k \leq n$. Note that the variables in the above chain are from different gadgets -- from both $\phi_k$ and from $\Gamma_{i,2}$.
This implies that $X''(w_{i,j})=1$  since otherwise $X''$ would have to set all variables in this chain to $0$ and then this would violate Lemma~\ref{clm1}. Since $X''$ is a model of
$T''$, it must be that $X''(z_i)=1$ for all $i$, otherwise $\neg z_i
\OR \neg w_{i,1} \OR \neg w_{i,2} \OR \neg w_{i,3}$ will be false.
When $z_i$ is flipped, we are guaranteed a repair of at most $s$ flips
that will make the clause $\neg z_i \OR \neg w_{i,1} \OR \neg w_{i,2}
\OR \neg w_{i,3}$ true.  This will involve flipping at least one of
$w_{i,j}$, for $j=1,2,3$.  If $v_{i,1}, v_{i,2}$ and $v_{i,3}$ were
all set to true by $X''$ (which would in turn have implied that
$X''(\alpha_{i,j,k})=1$ for all $1 \leq j \leq 3, 1 \leq k \leq s-1$)
then any such flip would require $s$ additional repairs, for a total
of $s+1$ repairs to a break to $z_i$. So it must be that 
$v_{i,j}$ is false for some $j, 1 \leq j \leq 3$. In other words, $C_i'= \neg v_{i,1} \OR \neg
v_{i,2} \OR \neg v_{i,3}$ is satisfied by  $X''$. Hence the restriction
of $X''$ to $T'$ satisfies all clauses of $T'$. Thus $T'$ is
satisfiable.

So $T'$ is satisfiable iff $T''$ has a $\delta(1,s)$-model. Since $T$
is satisfiable iff $T'$ is satisfiable and $T$ is a SAT instance, this
accomplishes the reduction from SAT. This reduction is clearly a
polynomial time reduction. Since $\de(1,s)$-Horn-SAT is clearly in NP for fixed $r$ and $s$, this proves
that it is NP-complete.
\end{proof}

\medskip

Recall that an dual-Horn formula is a Boolean formula in CNF where
each clause has at most one negative literal.  Not surprisingly,
dual-Horn-SAT formulas behave similarly to Horn-SAT when it comes to
finding $\de$-models.

\begin{theorem}
$\de(1,s)$-dual-Horn-SAT is NP-complete.
\end{theorem}

The proof of this theorem is very similar to that of Theorem~\ref{Hornb}: we replace Equation~(\ref{horneq1}) by
$ \Gamma_{i,1}  =  (z_i \OR  w_{i,1} \OR  w_{i,2} \OR   w_{i,3})$ and change the direction of implications in $\Gamma_{i,3}$
and Equation~(\ref{horneq2}).

\subsection{Finding $\de$-models  for $0$-valid,  $1$-valid SAT formulas}\label{subsection-valid}

Recall that a $0$-valid (resp. $1$-valid) Boolean formula is one which
is satisfied by a model with every variable set to $0$ (resp. $1$).
We now consider the complexity of finding fault-tolerant models of an
input $0$-valid (or $1$-valid) formula and refer to the corresponding
decision questions as $\delta(r,s)$-$0$-valid-SAT, $\delta(r,s)$-$1$-valid-SAT,
$\de^*$-$0$-valid-SAT etc.

The knowledge that an input Boolean formula is satisfied by some particular
assignment does not provide information about the presence of
fault-tolerant models.  Hence we would expect (correctly) that finding
such models to be NP-hard. We first prove:

\begin{theorem} \label{thr:1}
The decision problem $\de(r,s)$-0-valid-SAT is NP-complete.
\end{theorem}

\begin{proof}
  For the proof, we refer to the proof of Theorem~\ref{NPC} which, with
  slight modification, works for this problem as well.  We reduce from
  SAT. Let $T$ be a SAT instance, we construct an instance of
  $\de(r,s)$-0-valid-SAT, $T'=T \vee \neg y$ where $y$ is a new variable not
  appearing in $T$. Observe that $T'$ is $0$-valid (its the value of
  $y$ that matters).  The proof that $T'$ has a $\de$-model iff $T$ is
  satisfiable is identical to the proof of Theorem~\ref{NPC}: if $T$
  is satisfiable and has a model $X$, extend that to a model $X'$ of
  $T'$ by setting the value of $y$ to $1$.  Then any break consisting of
  $r$ variables in $X'$ does not require a repair if the $r$
  variables involve $y$.  If they do not involve $y$, then flipping
  the value of $y$ from a $1$ to a $0$ makes $T'$ true, hence one
  repair suffices.  Hence $X'$ is a $\de(r,s)$-model. If $T'$ has a
  $\de(r,s)$-model, it must have a model with $y$ set to $1$. The
  restriction of that model to the variables of $T$ makes $T$ true,
  hence $T$ is satisfiable.
\end{proof}

Similarly, it is easy to verify that the proofs of Theorem~\ref{NPC},
Theorem~\ref{superstar} work when the input formula is a $0$-valid or
$1$-valid formula.  Hence we have the following:

\begin{theorem}
The decision problem $\de(r,s)$-1-valid-SAT  is NP-complete. The problem $\de^*(1,1)$-0-valid-SAT and $\de^*(1,1)$-1-valid-SAT are in NEXP and are NP-hard.
\end{theorem}

\subsection{Finding $\de$-models for Affine-SAT}
\label{sec:finding-de-models-affine}

 Another class of Boolean formulas that have polynomial time
satisfiability checkers is  Affine-SAT: these are formulas
which are a conjunction of clauses, where each clause
is an exclusive-or (denoted by $\oplus$) of distinct literals
($a \oplus b=1$ iff exactly one of the Boolean variables $a,b$ is set to $1$).

\begin{example} \label{ex:5}
An example of an Affine-SAT formula is
\[ (x_1 \oplus x_2 \oplus x_3 \oplus x_4 =1) \AND (x_3 \oplus x_4 = 0) \]
This formula has a $\de$-model $X=(1,0,0,0)$. In fact, $X$ is easily seen to be a $\de^*$-model
(which is true of all $\de$-models of Affine-SAT formulas, as we shall shortly see).
\end{example}

One can find a satisfying assignment for a formula in affine form by a
variant of Gaussian elimination. We now prove that finding
$\delta$-models for affine formulas is also in polynomial time.

\begin{lemma} \label{afsat1} An Affine-SAT formula $\phi$ has a
  $\de$-model iff $\phi$ is satisfiable and for every variable $v \in
  V$ appearing in $\phi$ there exists a variable $w=w(v)$ such that
  $v$ and $w$ appear in exactly the same clauses.
  \end{lemma}

  \begin{proof}  
  Let $X$ be a $\de$-model of $\phi$.
  If a variable $v$ is flipped, then the clauses
  that $v$ appears in become false, to repair them we need to flip
  some other variable that appears in exactly those clauses (and no
  others). Thus such a variable pairing must exist. The reverse direction is easily proved:
  if such a variable pairing exists, then the variables form a break-repair pair.
\end{proof}

Since the conditions of Lemma~\ref{afsat1} are easy to check in polynomial time, we have the following theorem:

\begin{theorem}\label{affine-sat}
 $\de(1,1)$-Affine-SAT $\in$ P.
\end{theorem}

We can, in fact, slightly strengthen our theorem. We first state an analogue of Lemma~\ref{afsat1}, where
the variable pairings can be easily generalized.

\medskip

\noindent \textbf{Definition:} The parity of an integer $n$ is $n \bmod 2$.

\medskip

\begin{lemma}\label{afsat2}
    An Affine-SAT formula $\phi$ has a $\de(r,s)$-model iff $\phi$ is satisfiable and for every set $R$ of at most $r$ variables, there
exists a set $S, S \intersect R= \emptyset$ of at most $s$ variables, such that for all clauses $C$ of $\phi$, the parity of the number of variables of $R$ appearing in $C$ is the same as the parity of the number of variables of $S$ appearing in $C$.
 \end{lemma}

We now prove:

\begin{theorem}
  $\de(r,s)$-Affine-SAT is in $P$.
\end{theorem}

\begin{proof} Since $r$ and $s$ are fixed constants, the conditions in
  Lemma~\ref{afsat2} can be checked in polynomial time: for each
  choice of the set $R$ such that $R \leq r$, (there are $O(n^{r})$
  such sets), cycle through each possible set $S$ where $|S| \leq s, S
  \intersect R = \emptyset$ (there are $O(n^s)$ such sets), check to
  see if the conditions of Lemma~\ref{afsat2} are satisfied (in
  particular, test whether the parity of the variables of $R$
  appearing in any clause $=$ parity of the variables of $S$ appearing
  in the clause, which also can be accomplished in polynomial time).

Hence $\de(r,s)$-Affine-SAT is in polynomial time.  \end{proof}

Theorem~\ref{affine-sat} implies that any $\delta$-model of $\phi$
is actually a $\de^*$-model, since if the pairings $(u,w(v))$ exist, \emph{any model} of $\phi$
will become a $\de$-model (with $\{v, w(v)\}$ forming break-repair pairs).

Hence an Affine-SAT formula has a
$\de$-model iff it has a $\de^*$-model, hence finding a $\de^*$-model for Affine-SAT formulas 
is also in polynomial time.

We thus have the following theorem:
\begin{theorem}
$\de^*$-Affine-SAT $\in $ P.
\end{theorem}

\section{Future Work}\label{section-future}

    The complexity of $\delta(r,s)$-SAT where $r$ and $s$ are part of the
input as opposed to being fixed constants is not known. This problem
is in the complexity class $\Sigma_3^p$, but is it complete for that
class?  The status of this problem for restricted Boolean formulas like $2$-SAT, Horn-SAT etc., when $r$ and $s$ are specified in the input
is similarly open. At present, we do not also know if
$\delta^*(r,s)$-SAT can be decided in polynomial space when $r,s$ are
fixed constants.

Finally, a practical modification of the concept of $\de$-models would
involve weakening the condition to allow for only a high percentage of
breaks to be repairable.

\section*{Acknowledgements}

The author is grateful to  Eugene M. Luks for his encouragement and advice.  We also thank the anonymous
referees for their detailed comments and suggestions.

\bibliography{roy} \bibliographystyle{theapa}

\end{document}